\documentclass{article}
\pdfoutput=1
\usepackage{spconf,amsmath,graphicx}
\usepackage{amssymb}
\usepackage{bm}
\usepackage{mathdots}
\usepackage{algorithm}
\usepackage{algorithmic}
\usepackage{amsthm}
\usepackage{comment}
\usepackage{subcaption, multirow, hhline, array}
\usepackage{cite, url}
\usepackage{epstopdf}

\newcommand{\algname}{\textsc{DeepBeam }}
\newcommand{\algnamens}{\textsc{DeepBeam}}

\newtheorem{theorem}{Theorem}

\title{Deep Learning Based Speech Beamforming}
\name{Kaizhi Qian$^{1*}$, Yang Zhang$^{2*}$\thanks{$^*$ Denotes equal contribution.}, Shiyu Chang$^2$, Xuesong Yang$^1$, Dinei Florencio$^3$, Mark Hasegawa-Johnson${^1}$\thanks{This paper was funded by QNRF grant NPRP 7-766-1-140.}}
\address{$^1$University of Illinois at Urbana-Champaign, USA \\ $^2$ IBM T. J. Watson Research Center, USA \quad
  $^3$Microsoft Research, USA 
  \\{\small\tt \{kqian3,xyang45,jhasegaw\}@illinois.edu, \{yang.zhang2,shiyu.chang\}@ibm.com, dinei@microsoft.com  }}
\begin{document}
\ninept
\maketitle
\begin{abstract}
Multi-channel speech enhancement with ad-hoc sensors has been a challenging task. Speech model guided beamforming algorithms are able to recover natural sounding speech, but the speech models tend to be oversimplified or the inference would otherwise be too complicated. On the other hand, deep learning based enhancement approaches are able to learn complicated speech distributions and perform efficient inference, but they are unable to deal with variable number of input channels. Also, deep learning approaches introduce a lot of errors, particularly in the presence of unseen noise types and settings. We have therefore proposed an enhancement framework called \algnamens, which combines the two complementary classes of algorithms. \algname introduces a beamforming filter to produce natural sounding speech, but the filter coefficients are determined with the help of a monaural speech enhancement neural network. Experiments on synthetic and real-world data show that \algname is able to produce clean, dry and natural sounding speech, and is robust against unseen noise.
\end{abstract}
\section{Introduction}
\label{sec:intro}

Multi-channel speech enhancement with ad-hoc sensors has long been a challenging task  \cite{brandstein2013microphone}. As the traditional benchmark in multi-channel enhancement tasks, beamforming algorithms do not work well with with ad-hoc microphones. This is because most beamformers need to calibrate the speaker location as well as the interference characteristics, so that it can turn its beam toward the speaker, while suppressing the interference. However, neither of the two vital information can be accurately measured, due to the missing sensor position information and microphone heterogeneity \cite{markovich2015optimal}.

Another class of beamforming algorithms avoid measuring the speaker position and interference. Instead, they introduce prior knowledge on speech, and find the optimal beamformer by maximizing the "speechness" criteria, such as sample kurtosis \cite{gillespie2001speech}, negentropy \cite{kumatani2009beamforming}, speech prior distributions \cite{kim2007blind, kitamura2016determined}, fitting glottal residual \cite{zhang2017glottal} etc. In particular, the GRAB algorithm \cite{zhang2017glottal} is able to outperform the closest microphone strategy even in very adverse real-world scenarios. Despite their success, these algorithms are bottlenecked by their oversimplified prior knowledge. For example, GRAB only models glottal energy, resulting in vocal tract ambiguity.

On the other hand, deep learning techniques are well known for their ability to capture complex probability dependencies and efficient inference, and thus have been widely used in single-channel speech enhancement tasks \cite{chen2016long, huang2014deep, weninger2014discriminatively, qian2017speech, rethage2017wavenet, pascual2017segan}. 
Unfortunately, directly applying deep enhancement networks to multi-channel enhancement suffers from two difficulties.
First, deep enhancement techniques often produce a lot of artifacts and nonlinear distortions \cite{qian2017speech, rethage2017wavenet}， which are perceptually undesirable.
Second, neural networks often generalize poorly to unseen noise and configurations, whereas in speech enhancement with ad-hoc sensors, such variability is large.

It turns out, these problems can in turn be resolved by traditional beamforming. Therefore, several algorithms \cite{heymann2016neural,erdogan2016improved,xiao2017time,zhang2017speech,pfeifenberger2017dnn} have been proposed that applies deep learning to predict time-frequency masks, and then beamforming to produce the enhanced speech. However, these methods are confined to frequency domain, which suffers from two problems for our application. First, they to not work well for ad-hoc microphones, because of the spatial correlation estimation errors. Second, our application is for human consumption, but the frequency-domain methods suffer from phase distortions and discontinuities, which impede perceptual quality.

Motivated by this observation, we have proposed an enhancement framework for ad-hoc microphones called \algnamens, which combines deep learning and beamforming, and which directly works on waveform. \algname introduces a time-domain beamforming filter to produce natural sounding speech, but the filter coefficients are iteratively determined with the help of WaveNet \cite{oord2016wavenet}. It can be shown that despite the error-prone enhancement network, \algname is able to converge approximately to the optimal beamformer under some assumptions. Experiments on both the simulated and real-world data show that \algname is able to produce clean, dry and natural sounding speech, and generalize well to various settings.

\section{Problem Formulation}
\label{sec:formulation}
To formally define the problem, denote $s[t]$ as the clean speech signal. Suppose there are $K$ channels of observed signals, $y_k[t], k = 1, \cdots, K$, which are represented as
\begin{equation}
\small
y_k[t] = s[t] * i_k[t] + n[t] * j_k[t]
\label{eq:noisy_obs}
\end{equation}
where $*$ denotes discrete convolution, $n(t)$ denotes additive noise. $i_k[t]$ and $j_k[t]$ are the impulse responses of the signal reverberation and noise reverberation in the $k$-th channel respectively. Our goal is to design a $\tau$-tap beamformer $h_k[t], k = 1, \cdots, K$, whose output is defined as
\begin{equation}
\small
x[t] = \sum_{k = 1} ^ K y_k[t] * h_k[t]
\label{eq:beamform_output}
\end{equation}

For notational brevity, define
\begin{equation}
\small
\begin{aligned}
& \bm{s} = [s[1], \cdots, s[T]]^T \quad \bm x = [x[1], \cdots, x[T]]^T\\
& \bm{y}_k = [y_k[1], \cdots, y_k[T]]^T \quad \bm{y} = [\bm y_1^T, \cdots, \bm y_K^T]^T\\
& \bm{h} = [h_1[1], \cdots, h_1[\tau], h_2[1], \cdots, h_K[\tau]]^T
\end{aligned}
\end{equation}
which are all random vectors. Also define convolutional matrices
\begin{equation}
\small
\bm Y_k = \left[
\begin{array}{cccc}
y_k[1] & & & \\ 
y_k[2] & y_k[1] & &  \\
\vdots & \vdots & \ddots & \\
y_k[\tau] & y_k[\tau-1] & \cdots & y_k[1] \\
\vdots & \vdots &  & \vdots \\
y_k[T] & y_k[T-1]  & \cdots & y_k[T - \tau + 1]
\end{array}
\right]
\end{equation}
and
\begin{equation}
\small
\bm Y = [\bm Y_1, \cdots, \bm Y_K]
\end{equation}

With these notations, Eq.~\eqref{eq:beamform_output} can be simplified as
\begin{equation}
\small
\bm x = \bm Y \bm h
\end{equation}

The target of designing the beamformer is to minimize the weighted mean squared error (MSE):
\begin{equation}
\small
\min_{\bm x = \bm Y \bm h} \mathbb{E} \left[ \lVert \bm x - \bm s \rVert_{\bm W}^2| \bm y\right]
\label{eq:wiener_prob}
\end{equation}
where $\lVert \bm x \rVert_{\bm W}^2 = \bm x^T \bm W \bm x$; $\bm W$ is a positive definite weight matrix, which, in our case, is a diagonal matrix of $\mbox{Var}^{-1}(s[t] | \bm y)$.

Eq.~\eqref{eq:wiener_prob} is a Wiener filtering problem \cite{wiener1949extrapolation}, whose solution is
\begin{equation}
\small
\bm x^* = \bm P \mathbb{E}[\bm s | \bm y]
\label{eq:wiener_solution}
\end{equation}
where
\begin{equation}
\small
\bm P = \bm Y (\bm Y^T \bm W \bm Y)^{-1} \bm Y^T \bm W
\label{eq:projection}
\end{equation}
is in fact the \emph{projection matrix} onto the beamforming output space. So by Eq.~\eqref{eq:wiener_solution}, $\bm x^*$ is essentially projecting $\mathbb{E}[\bm s | \bm y]$ onto the space that is representable by the beamforming filter.

As shown by Eq.~\eqref{eq:wiener_solution}, solving the Wiener filtering problem requires computing $\mathbb{E}[\bm s | \bm y]$, which, due to the complex probabilistic dependencies, we would like to introduce a deep neural network to learn. However, as discussed, training a neural network to directly predict $\mathbb{E}[\bm s | \bm y]$ from the multi-channel input $\bm y$ suffers from inflexible input dimensions, artifacts and poor generalization. \algname tries to resolve these problems and find an approximate solution.

\section{The \algname Framework}
\label{sec:framework}
In this section, we will describe the \algname algorithm. We will first outline the algorithm, and then describe the neural network structure it applies. Finally, a convergence analysis is introduced.

\subsection{The Algorithm Overview}
\label{subsec:overview}

As mentioned, \algname introduces a deep enhancement network to learn the posterior expectation, while addressing its limitations. First, \algname are regularized by the beamformer to generalize well to unseen noise and microphone configurations. Second, it tolerates the distortions and artifacts generated by the neural network. Formally, the neural network outputs an inaccurate prediction of the posterior expectation $\mathbb{E}[\bm s | \bm \xi]$, 
\begin{equation}
\small
f(\bm \xi) = \mathbb{E}[\bm s | \bm \xi] + \bm \varepsilon(\bm \xi)
\label{eq:inaccurate_posterior}
\end{equation}
where $\bm \xi$ is a \emph{single-channel} noisy observation, and $\bm \varepsilon(\bm \xi)$ is the prediction error. The goal of \algname is to approximate the optimal beamformer given the inaccurate enhancement network. Alg.~\ref{alg:alg} shows the description of the \algname algorithm. A graph of the \algname framework is shown in Fig.~\ref{fig:framework}.
\begin{figure}
	\centering
	\includegraphics[width = 1\linewidth]{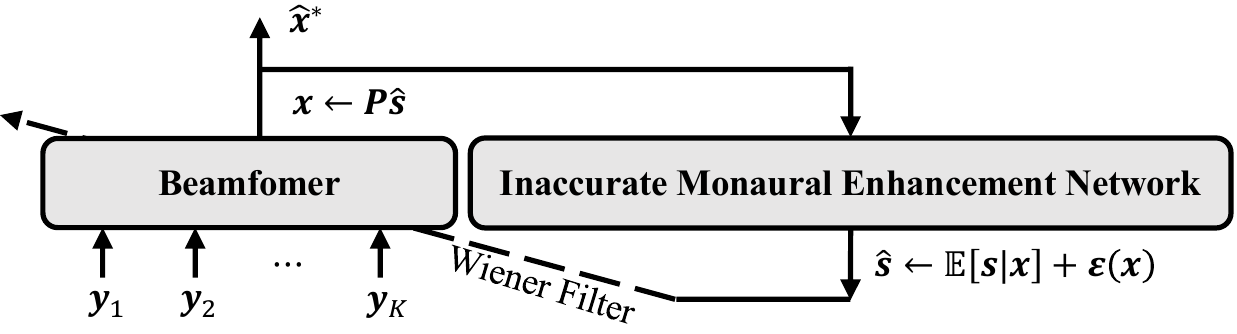}
	\caption{\algname framework.}
	\label{fig:framework}
\end{figure}

\begin{algorithm}[!t]
\caption{The \algname algorithm.}
\begin{algorithmic}[1]
\REQUIRE Multi-channel noisy speech observations $\bm y$; \\
A neural network that predicts $f(\bm \xi)$ (Eq.~\eqref{eq:inaccurate_posterior}) from any single-channel noisy observation $\bm \xi$.
\ENSURE Beamformer output $\hat{\bm x}^*$.
\newline\newline
\textbf{Initialization:}
\STATE Find the `cleanest' channel $k^*$ by finding the channel that has the smallest 0.4 quantile of its squared sample points.
\STATE Set $\bm x^{(0)} = \bm y_{k^*}$.
\newline
\textbf{Iteration:}
\FOR{$n = 1$ to maximum number of iterations}
\STATE Feed $\bm x^{(n-1)}$ to the monaural enhancement network, and obtain its output
\begin{equation}
\small
\hat{\bm s}^{(n)} = f(\bm x^{(n-1)})= \mathbb{E}[\bm s | \bm x^{(n-1)}] + \bm \varepsilon(\bm x^{(n-1)})
\label{eq:enhance_net_update}
\end{equation}
\vspace{-0.2in}
\STATE Update the beamformer coefficients and output
\begin{equation}
\small
\bm x^{(n)} = \bm P \hat{\bm s}^{(n)}
\label{eq:wiener_update}
\end{equation}
\ENDFOR
\RETURN{$\hat{\bm x}^* = \bm x^{(N)}$}
\end{algorithmic}
\label{alg:alg}
\end{algorithm}
\vspace{-0.1in}

Alg.~\ref{alg:alg} essentially alternates between the posterior expectation and projection iteratively. It will be shown in section~\ref{subsec:analysis} that as long as the error term $\bm \varepsilon$ is not too large, this iteration will approximately converge to the optimal beamformer output.

One elegance of \algname is that $\bm x^{(n)}$ can be regarded as a noisy observation, and shares some statistical structures with the true noisy observations, $\bm y_k$. To see this, notice that by Eq.~\eqref{eq:wiener_update}, $\bm x^{(n)}$ is the output of a beamformer on $\bm y$. Therefore, it can be shown that $\bm x^{(n)}$ also takes the form of Eq.~\eqref{eq:noisy_obs}, with the same speech and noise source, but with a different impulse response. This justifies the use of one monaural enhancement network to take care of all the $\bm x^{(n)}$.

\subsection{Enhancement Network Structure}
\label{subsec:enhance_net}

\algname is a general framework, in which the choice of the neural network structure is not fixed. The following network structure is just one of the structures that produce competitive results.

The enhancement network applied here is similar to \cite{rethage2017wavenet}, which is inspired by WaveNet \cite{oord2016wavenet}. Formally, denote the \emph{quantized} speech samples as $\tilde{s}[t]$, and the samples of $\bm x^{(n)}$ as $x^{(n)}[t]$. Then the enhancement network predicts the posterior probability mass function (PMF) of $\tilde{s}[t]$:
\begin{equation}
\small
p(\tilde{s}[t] | \bm x^{(n)}) \approx p(\tilde{s}[t] | x^{(n)}[t - \tau_r], \cdots, x^{(n)}[t + \tau_r])
\end{equation}
Here we have restricted the probabilistic dependency to span $\tau_r$ time steps. Cross-entropy is applied as the loss function.

Similar to WaveNet, the enhancement network consists of two modules. The first module, called the dilated convolution module, contains a stack of dilated convolutional layers with residual connections and skip outputs. The second module, called the post processing module, sums all the skip outputs and feeds them into a stack of fully connected layers before producing the final output.

There are two major differences from the standard WaveNet structure. First, the input to the enhancement network is the noisy observation waveform $\bm x^{(n)}$ instead of the clean speech. Second, to account for the future dependencies, the convolutional layers are noncausal $1 \times 3$ instead of the causal $1 \times 2$.

After the posterior distribution is predicted, the posterior moments, $\bm E[\bm s | \bm x^{(n)}]$ and $\mbox{Var}[s[t] | \bm y]$ (for computing $\bm W$), are computed as the moments of the predicted PMF.

\subsection{Convergence Analysis}
\label{subsec:analysis}

In order to analyze the convergence property of \algnamens, we assume the following bound on the error term
\begin{equation}
\small
\mathbb{E}[  \lVert \bm P \bm \varepsilon (\bm x^{(n)})  \rVert^2 _{\bm W} | \bm y] \leq
\rho \mathbb{E}[ \lVert \bm x^{(n)} - \bm s \rVert^2 _{\bm W} | \bm y]
\label{eq:limited_error}
\end{equation}
where $\rho < 0.5$ is some constant. This assumption is actually not quite stringent, because it does not bound the weighted norm of $\bm \varepsilon(\bm x^{(n)})$ itself, but its projected value $\bm P \bm \varepsilon(\bm x^{(n)})$. In fact, the projection can drastically reduce the weighted norm of the error term. For example, most of the artifacts and nonlinear distortions that the enhancement network introduces cannot possibly be generated by beamforming on $\bm y$, and therefore will be removed by the projection. The only errors that are likely to remain are residual noise and reverberations. This is one advantage of combining beamforming filter and neural network. This assumption is also very intuitive. It means that the projected output error is always smaller than input error. 

Then, we have the following theorem.

\begin{theorem} Suppose Eq.~\eqref{eq:limited_error} holds. Then
\begin{equation}
\small
\underset{n \rightarrow \infty}{\operatorname{lim~sup}} \mathbb{E}[ \lVert  \bm x^{(n)} - \bm x^*  \rVert^2 _{\bm W} | \bm y] \leq u
\label{eq:convergence}
\end{equation}
where
\begin{equation}
\small
\begin{aligned}
u &= \frac{2\rho}{1 - 2\rho} \mathbb{E}[ \lVert \bm s - \bm x^* \rVert^2 _{\bm W} | \bm y] \\
&+ \frac{2}{1 - 2\rho} \sup_{n} \mathbb{E} [\lVert \bm P \mathbb{E}[\bm s | \bm x^{(n)}] - \bm x^* \rVert^2_{\bm W} | \bm y ]
\end{aligned}
\end{equation}
\label{thm}
\end{theorem}
\begin{proof}
On one hand, from Eqs.~\eqref{eq:enhance_net_update} and \eqref{eq:wiener_update}
\begin{equation}
\small
\begin{aligned}
&\mathbb{E} [ \lVert \bm P \bm \varepsilon(\bm x ^{(n)}) \rVert^2_{\bm W} | \bm y]= 
\mathbb{E} [ \lVert \bm x^{(n+1)} - \bm P \mathbb{E}[\bm s | \bm x^{(n)}] \rVert^2_{\bm W}| \bm y ] \\
\geq & \frac{1}{2}\mathbb{E} [\lVert \bm x^{(n+1)} - \bm x^*\rVert_{\bm W}^2 | \bm y]  - \mathbb{E} [ \lVert \bm P \mathbb{E}[\bm s | \bm x^{(n)}] - \bm x^* \rVert^2_{\bm W} | \bm y ]
\end{aligned}
\label{eq:lhs}
\end{equation}
On the other hand, by orthogonality principle
\begin{equation}
\small
\begin{aligned}
&\mathbb{E}[  \lVert \bm x^{(n)} - \bm s  \rVert^2 _{\bm W} | \bm y] 
= \mathbb{E}[ \lVert \bm x^{(n)} - \bm x^* \rVert^2 _{\bm W} | \bm y]
+ \mathbb{E}[ \lVert \bm s - \bm x^* \rVert^2 _{\bm W} | \bm y]
\end{aligned}
\label{eq:rhs}
\end{equation}
Combining Eqs.~\eqref{eq:limited_error}, \eqref{eq:lhs} and \eqref{eq:rhs}, we have
\begin{equation}
\small
\begin{aligned}
& \mathbb{E}[ \lVert \bm x^{(n+1)} - \bm x^* \rVert^2 _{\bm W} | \bm y]
\leq  2\rho \mathbb{E}[ \lVert \bm x^{(n)} - \bm x^* \rVert^2 _{\bm W} | \bm y] + (1 - 2\rho) u
\end{aligned}
\label{eq:recursive}
\end{equation}
Create an auxiliary sequence
\begin{equation}
\small
a^{(n)} = \mathbb{E}[ \lVert  \bm x^{(n)} - \bm x^* \rVert^2 _{\bm W} | \bm y] - u
\end{equation}
\small
Then by Eq.~\eqref{eq:recursive},
\begin{equation}
a^{(n+1)} \leq (2\rho)^n a^{(1)}
\label{eq:a_recursion}
\end{equation}
Taking $\mbox{lim sup}_{n\rightarrow\infty}$ on both sides of Eq.~\eqref{eq:a_recursion} concludes the proof.
\end{proof}

If $u = 0$, then Eq.~\eqref{eq:convergence} implies mean square convergence to the optimal beamformer output. In actuality, $u$ is nonzero, but it tends to be very small. The first term of $u$ measures the distance between the optimal beamformer output and the true speech. According to our empirical study, when the number of channel is sufficient, the optimal beamformer is able to recover the true speech very well, so the first term is small. The second term of $u$ measures the distance between two posterior expectations $\bm P \mathbb{E}[\bm s | \bm x^{(n)}]$ and $\bm P \mathbb{E}[\bm s | \bm y]$. The former is conditional on single-channel noisy speech, and the latter on multiple-channel noisy speech. Considering that the speech sample space is highly structured, and that the noisy speech $\bm x^{(n)}$ is relatively clean already, both posterior expectations should be close to the true speech, and thereby close to each other. In a nutshell, with a small $u$, the \algname prediction is highly accurate. Section~\ref{subsec:convergence} will verify the convergence behavior of \algname empirically.

\section{Experiments}
\label{sec:exper}
This section first introduces how the enhancement network is configured and trained, and then presents the results of experiments on both simulated and real-world data. Audio samples can be found in \url{http://tiny.cc/a1qjoy} .

\subsection{Enhancement Network Configurations}
\label{subsec:net_config}
The enhancement network hyperparameter configurations follow \cite{oord2016wavenet}. It has 4 blocks of 10 dilated convolution layers. There are two post processing layers. The hidden node dimension is 32, and the skip node dimension is 256. The clean speech is quantized into 256 level via $\mu$-law companding, and thus the output dimension is 256. The activation function in the dilated convolutional layers is the gated activation unit; that in the post processing layers is the ReLU function. The output activation is softmax.

The enhancement network is trained on simulated data \emph{only}, which is generated in the same way as in \cite{zhang2017glottal}. The speech source, noise source and eight microphones are randomly placed into a randomly sized cubic room. The impulse response from each source to each microphone is generated using the image-source method \cite{allen1979image, lehmann2010diffuse}. The noisy observations are generated according to Eq.~\eqref{eq:noisy_obs}. The reverberation time is uniformly randomly drawn from [$100$, $300$]~ms. The energy ratio between the speech source and noise source, $E_r$, is uniformly randomly drawn from [$-5$, $20$]~dB. The speech content is drawn from VCTK \cite{yamagishienglish}, which contains 109 speakers. The noise content contains 90 minutes of audio drawn from \cite{kumar2016speech,soundcloud,hu100}. The total duration of the training audio is 8 hours. The enhancement network is trained using ADAM optimizer for 400,000 iterations.

\subsection{Simulated Data Evaluation}
\label{subsec:simulated}

The simulated data for evaluation is generated the same way as the training data, except for two differences. First, the source energy ratio, $E_r$, is set to four levels, $-10$~dB, $0$~dB, $10$~dB, and $20$~dB. Second, both the speaker and noise can be either seen or unseen in the training set, leading to four different scenarios to test generalizability. It is worth highlighting that the unseen speaker utterances and unseen noise are both drawn from different corpora from training, TIMIT \cite{garofolo1993darpa} and FreeSFX \cite{FreeSFX} respectively. Each utterance is 3 seconds in length. The total length of the dataset is 12 minutes.

\algname is compared with GRAB \cite{zhang2017glottal}, MVDR\footnote{Clean speech is given for voice activity detection.} \cite{griffiths1982alternative}, IVA \cite{kim2007blind} and the closest channel (CLOSEST), in term of two criteria:\\
\noindent
$\bullet\quad$ {\bf Signal-to-Noise Ratio (SNR)}: The energy ratio of processed clean speech over processed noise in dB.

\noindent
$\bullet\quad$ {\bf Direct-to-Reverberant Ratio (DRR)}: the ratio of the energy of direct path speech in the processed output over that of its reverberation in dB. Direct path and reverberation are defined as clean dry speech convolved with the peak portion and tail portion of processed room impulse response. The peak portion is defined as $\pm6$~ms within the highest peak; the tail portion is defined as $\pm6$~ms beyond.

\begin{table}
\caption{Simulated Data Evaluation Results.}
\label{tab:simulated}
	\small
	\centering
	\begin{tabular}{l| l | c c c c}
    	\hline\hline
        \multicolumn{2}{c|}{$E_r = $} & -10 & 0 & 10 & 20 \\
        \hline
        \multirow{8}{0.7cm}{\textbf{SNR} (dB)} & \algname S1 & 18.5 & 22.0 & 26.5 & 28.4\\
         & \algname S2 & 17.1 & 20.3 & 25.9 & 27.4 \\
         & \algname S3 & 15.3 & 19.5 & 24.1 & 27.6\\
         & \algname S4 & 14.1 & 19.0 & 23.1 & 28.5\\
         & GRAB & 2.48 & 12.5 & 21.6 & 25.4 \\
         & CLOSEST & -5.13 & 3.38 & 14.9 & 24.8 \\
         & MVDR & 8.41 & 12.9 & 22.6 & 26.7 \\
         & IVA & 10.3 & 13.3 & 16.8 & 19.2\\
         \hline
        \multirow{8}{0.7cm}{\textbf{DRR} (dB)} & \algname S1 & 3.45 & 8.97 & 11.2 & 11.5 \\
         & \algname S2 & 7.38 & 11.9 & 12.6 & 11.5 \\
         & \algname S3 & 5.60 & 4.85 & 8.43 & 9.78 \\
         & \algname S4 & 2.11 & 6.68 & 7.10 & 9.31 \\
         & GRAB & -0.83 & 1.70 & 3.63 & 3.68 \\
         & CLOSEST & 8.56 & 7.32 & 7.67 & 8.44 \\
         & MVDR & -2.17 & -3.47 & -3.42 & -4.13 \\
         & IVA & -8.92 & -8.77 & -8.81 & -8.99 \\
		\hline\hline
        \multicolumn{6}{l}{\scriptsize{S1: seen speaker, seen noise;~~~~~~~~S2: seen speaker, unseen noise;}} \\
        \multicolumn{6}{l}{\scriptsize{S3: unseen speaker, seen noise;~~~~S4: unseen speaker, unseen noise.}}
	\end{tabular}
    \vspace{-0.1in}
\end{table}

Table~\ref{tab:simulated} shows the results. As expected, \algnamens's performance drops from S1, where both noise and speaker are seen during training, to S4, where neither is seen. However, in terms of SNR, even \algname S4 significantly outperforms MVDR, which is the benchmark in noise suppression. In terms of DRR, \algname matches or surpasses CLOSEST except for -10 dB. GRAB performs poorer than in \cite{zhang2017glottal}, because each utterance is reduced from 10 seconds to 3 seconds, which is more realistic but challenging. In short, of ``cleanness'' and ``dryness'', most algorithms can only achieve one, but \algname can achieve \emph{both} with superior performance.

\subsection{Real-world Data Evaluation}
\label{subsec:realworld}

\algname and the baselines are also evaluated on the real-world dataset introduced in \cite{zhang2017glottal}, which consists of two utterances by two speakers mixed with five types of noises, all recorded in a real conference room using eight randomly positioned microphones. The source energy ratio is set such that the SNR for the closest microphone is $10$~dB. The utterance in each scenario is around 1 minute long, so the total length of the dataset is $10$ minutes.

Besides SNR, a subjective test similar to \cite{zhang2017glottal} is performed on Amazon Mechanical Turk. Each utterance is broken into six sentences. In each test unit, called HIT, a subject is presented with one sentence processed by the five algorithms, and asked to assign an MOS \cite{ribeiro2011crowdmos} to each of them. Each HIT is assigned to 10 subjects.

\begin{table}
\caption{Realworld Data Evaluation Results.}
\label{tab:realworld}
\centering
	\begin{tabular}{l| l | c c c c c}
    	\hline\hline
        \multicolumn{2}{c|}{\textbf{Noise Type}} & N1 & N2 & N3 & N4 & N5 \\
        \hline
        \multirow{5}{0.7cm}{\textbf{SNR} (dB)} & \algname & \textbf{20.1} & \textbf{20.0} & \textbf{16.9} & \textbf{19.6} & \textbf{18.7} \\
         & GRAB & 18.9 & 17.4 & 12.4 & 18.5 & 17.4 \\
         & CLOSEST & 10.0 & 10.0 & 10.0 & 10.0 & 10.0\\
         & MVDR & 10.8 & 16.5 & 7.72 & 14.0 & 13.4 \\
         & IVA & 11.7 & 9.74 & 6.83 & 12.4 & 15.9 \\
         \hline
        \multirow{5}{0.7cm}{\textbf{MOS}} & \algname & \textbf{3.83} & \textbf{3.72} & \textbf{3.63} & \textbf{4.09} & \textbf{4.20} \\
         & GRAB & 3.10 & 3.06 & 2.93 & 3.71 & 3.45 \\
         & CLOSEST & 2.74 & 2.68 & 3.02 & 3.55 & 3.50 \\
         & MVDR & 2.05 & 2.40 & 2.28 & 2.71 & 2.62 \\
         & IVA & 1.73 & 2.03 & 1.75 & 1.78 & 2.08 \\
		\hline\hline
        \multicolumn{7}{l}{\scriptsize{N1: cell phone;~~~~N2: CombBind machine;~~~~N3:paper shuffle;}} \\
        \multicolumn{7}{l}{\scriptsize{N4: door slide;~~~~~N5: footsteps.}}
	\end{tabular}
    \vspace{-0.2in}
\end{table}

Table~\ref{tab:realworld} shows the results. As can be seen, \algname outperforms the other algorithms by a large margin. In particular, \algname achieves $>4$ MOS in some noise types. These results are very impressive, because \algname is only trained on simulated data. The real-world data differ significantly from the simulated data in terms of speakers, noise types and recording environment. What's more, some microphones are contaminated by strong electric noise, which is not accounted for in Eq.~\eqref{eq:noisy_obs}. Still, \algname manages to conquer all the unexpected. Neural network used to be vulnerable to unseen scenarios, but \algname has now made it robust.

\subsection{Empirical Convergence Analysis}
\label{subsec:convergence}
\begin{figure}
	\centering
	\includegraphics[width = 1\linewidth]{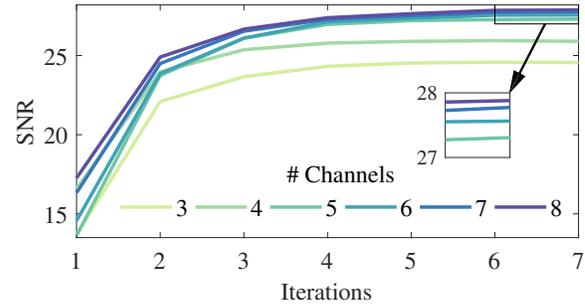}
    \vspace{-0.2in}
	\caption{SNR convergence curves with different numbers of channels.}
    \vspace{-0.1in}
	\label{fig:convergence}
\end{figure}

In order to empirically test whether \algname has a good convergence property, 10 sets of eight-channel simulated data are generated with the S1 setting and $E_r=10$. To study different number of channels, in each sub-test, $K$ channels are randomly drawn from each set of data for \algname prediction, and the resulting SNR convergence curves of the 10 sets are averaged. $K$ runs from 3 to 8.

Fig.~\ref{fig:convergence} shows all the averaged convergence curves. As can be seen, \algname converges well in all the sub-tests, which supports our convergence discussions in section~\ref{subsec:analysis}. Also, the more channels \algname has, the higher convergence level it can reach, which shows that \algname is able to accommodate different numbers of channels using only one monaural network. We also see that the marginal benefit of having one more channel diminishes.

\section{Conclusion}
\label{sec:conclu}
We have proposed \algname as a solution to multi-channel speech enhancement with ad-hoc sensors. \algname combines the complementary beamforming and deep learning techniques, and has exhibited superior performance and generalizability in terms of noise suppression, reverberation cancellation and perceptual quality. \algname has made one step closer to resolving the long lasting crux of low perceptual quality and poor generalizability in deep enhancement networks, which demonstrates the power of bridging the signal processing and deep learning areas.

\bibliographystyle{IEEE}
\small
\bibliography{refs}

\begin{thebibliography}{10}

\bibitem{brandstein2013microphone}
Michael Brandstein and Darren Ward,
\newblock {\em Microphone arrays: signal processing techniques and
  applications},
\newblock Springer Science \& Business Media, 2013.

\bibitem{markovich2015optimal}
Shmulik Markovich-Golan, Alexander Bertrand, Marc Moonen, and Sharon Gannot,
\newblock ``Optimal distributed minimum-variance beamforming approaches for
  speech enhancement in wireless acoustic sensor networks,''
\newblock {\em Signal Processing}, vol. 107, pp. 4--20, 2015.

\bibitem{gillespie2001speech}
Bradford~W Gillespie, Henrique~S Malvar, and Dinei~AF Flor{\^e}ncio,
\newblock ``Speech dereverberation via maximum-kurtosis subband adaptive
  filtering,''
\newblock in {\em IEEE International Conference on Acoustics, Speech and Signal
  Processing (ICASSP)}. IEEE, 2001, vol.~6, pp. 3701--3704.

\bibitem{kumatani2009beamforming}
Kenichi Kumatani, John McDonough, Barbara Rauch, Dietrich Klakow, Philip~N
  Garner, and Weifeng Li,
\newblock ``Beamforming with a maximum negentropy criterion,''
\newblock {\em IEEE Transactions on Audio, Speech, and Language Processing},
  vol. 17, no. 5, pp. 994--1008, 2009.

\bibitem{kim2007blind}
Taesu Kim, Hagai~T Attias, Soo-Young Lee, and Te-Won Lee,
\newblock ``Blind source separation exploiting higher-order frequency
  dependencies,''
\newblock {\em IEEE Transactions on Audio, Speech, and Language Processing},
  vol. 15, no. 1, pp. 70--79, 2007.

\bibitem{kitamura2016determined}
Daichi Kitamura, Nobutaka Ono, Hiroshi Sawada, Hirokazu Kameoka, and Hiroshi
  Saruwatari,
\newblock ``Determined blind source separation unifying independent vector
  analysis and nonnegative matrix factorization,''
\newblock {\em IEEE Transactions on Audio, Speech, and Language Processing},
  vol. 24, no. 9, pp. 1626--1641, 2016.

\bibitem{zhang2017glottal}
Yang Zhang, Dinei Flor{\^e}ncio, and Mark Hasegawa-Johnson,
\newblock ``Glottal model based speech beamforming for ad-hoc microphone
  arrays,''
\newblock {\em INTERSPEECH}, pp. 2675--2679, 2017.

\bibitem{chen2016long}
Jitong Chen and Deliang Wang,
\newblock ``Long short-term memory for speaker generalization in supervised
  speech separation,''
\newblock in {\em INTERSPEECH}, 2016, pp. 3314--3318.

\bibitem{huang2014deep}
Po-Sen Huang, Minje Kim, Mark Hasegawa-Johnson, and Paris Smaragdis,
\newblock ``Deep learning for monaural speech separation,''
\newblock in {\em IEEE International Conference on Acoustics, Speech and Signal
  Processing (ICASSP)}, 2014, pp. 1562--1566.

\bibitem{weninger2014discriminatively}
Felix Weninger, John~R Hershey, Jonathan Le~Roux, and Bj{\"o}rn Schuller,
\newblock ``Discriminatively trained recurrent neural networks for
  single-channel speech separation,''
\newblock in {\em IEEE Global Conference on Signal and Information Processing
  (GlobalSIP)}, 2014, pp. 577--581.

\bibitem{qian2017speech}
Kaizhi Qian, Yang Zhang, Shiyu Chang, Xuesong Yang, Dinei Flor{\^e}ncio, and
  Mark Hasegawa-Johnson,
\newblock ``Speech enhancement using {Bayesian} {Wavenet},''
\newblock {\em INTERSPEECH}, pp. 2013--2017, 2017.

\bibitem{rethage2017wavenet}
Dario Rethage, Jordi Pons, and Xavier Serra,
\newblock ``A {Wavenet} for speech denoising,''
\newblock {\em arXiv preprint arXiv:1706.07162}, 2017.

\bibitem{pascual2017segan}
Santiago Pascual, Antonio Bonafonte, and Joan Serr{\`a},
\newblock ``{SEGAN}: Speech enhancement generative adversarial network,''
\newblock {\em arXiv preprint arXiv:1703.09452}, 2017.

\bibitem{heymann2016neural}
Jahn Heymann, Lukas Drude, and Reinhold Haeb-Umbach,
\newblock ``Neural network based spectral mask estimation for acoustic
  beamforming,''
\newblock in {\em Acoustics, Speech and Signal Processing (ICASSP), 2016 IEEE
  International Conference on}. IEEE, 2016, pp. 196--200.

\bibitem{erdogan2016improved}
Hakan Erdogan, John~R Hershey, Shinji Watanabe, Michael~I Mandel, and Jonathan
  Le~Roux,
\newblock ``Improved mvdr beamforming using single-channel mask prediction
  networks.,''
\newblock in {\em INTERSPEECH}, 2016, pp. 1981--1985.

\bibitem{xiao2017time}
Xiong Xiao, Shengkui Zhao, Douglas~L Jones, Eng~Siong Chng, and Haizhou Li,
\newblock ``On time-frequency mask estimation for mvdr beamforming with
  application in robust speech recognition,''
\newblock in {\em Acoustics, Speech and Signal Processing (ICASSP), 2017 IEEE
  International Conference on}. IEEE, 2017, pp. 3246--3250.

\bibitem{zhang2017speech}
Xueliang Zhang, Zhong-Qiu Wang, and DeLiang Wang,
\newblock ``A speech enhancement algorithm by iterating single-and
  multi-microphone processing and its application to robust asr,''
\newblock in {\em Acoustics, Speech and Signal Processing (ICASSP), 2017 IEEE
  International Conference on}. IEEE, 2017, pp. 276--280.

\bibitem{pfeifenberger2017dnn}
Lukas Pfeifenberger, Matthias Z{\"o}hrer, and Franz Pernkopf,
\newblock ``Dnn-based speech mask estimation for eigenvector beamforming,''
\newblock in {\em Acoustics, Speech and Signal Processing (ICASSP), 2017 IEEE
  International Conference on}. IEEE, 2017, pp. 66--70.

\bibitem{oord2016wavenet}
Aaron van~den Oord, Sander Dieleman, Heiga Zen, Karen Simonyan, Oriol Vinyals,
  Alex Graves, Nal Kalchbrenner, Andrew Senior, and Koray Kavukcuoglu,
\newblock ``{WaveNet}: A generative model for raw audio,''
\newblock {\em arXiv preprint arXiv:1609.03499}, 2016.

\bibitem{wiener1949extrapolation}
Norbert Wiener,
\newblock {\em Extrapolation, interpolation, and smoothing of stationary time
  series}, vol.~7,
\newblock MIT press Cambridge, MA, 1949.

\bibitem{allen1979image}
Jont~B Allen and David~A Berkley,
\newblock ``Image method for efficiently simulating small-room acoustics,''
\newblock {\em The Journal of the Acoustical Society of America}, vol. 65, no.
  4, pp. 943--950, 1979.

\bibitem{lehmann2010diffuse}
Eric~A Lehmann and Anders~M Johansson,
\newblock ``Diffuse reverberation model for efficient image-source simulation
  of room impulse responses,''
\newblock {\em IEEE Transactions on Audio, Speech, and Language Processing},
  vol. 18, no. 6, pp. 1429--1439, 2010.

\bibitem{yamagishienglish}
Junichi Yamagishi,
\newblock ``English multi-speaker corpus for {CSTR} voice cloning toolkit,''
  \url{http://homepages.inf.ed.ac.uk/jyamagis/page3/page58/page58.html}.

\bibitem{kumar2016speech}
Anurag Kumar and Dinei Flor{\^e}ncio,
\newblock ``Speech enhancement in multiple-noise conditions using deep neural
  networks,''
\newblock {\em INTERSPEECH}, 2016.

\bibitem{soundcloud}
``Freesound,'' \url{https://freesound.org/}, 2015.

\bibitem{hu100}
Guoning Hu,
\newblock ``100 nonspeech sounds,''
  \url{http://web.cse.ohio-state.edu/pnl/corpus/HuNonspeech/HuCorpus.html},
  2015.

\bibitem{garofolo1993darpa}
John~S Garofolo, Lori~F Lamel, William~M Fisher, Jonathon~G Fiscus, and David~S
  Pallett,
\newblock ``{DARPA TIMIT} acoustic-phonetic continous speech corpus {CD-ROM}.
  {NIST} speech disc 1-1.1,''
\newblock {\em NASA STI/Recon technical report n}, vol. 93, 1993.

\bibitem{FreeSFX}
``{FreeSFX},'' \url{http://www.freesfx.co.uk/}, 2017.

\bibitem{griffiths1982alternative}
Lloyd Griffiths and CW~Jim,
\newblock ``An alternative approach to linearly constrained adaptive
  beamforming,''
\newblock {\em IEEE Transactions on antennas and propagation}, vol. 30, no. 1,
  pp. 27--34, 1982.

\bibitem{ribeiro2011crowdmos}
Fl{\'a}vio Ribeiro, Dinei Flor{\^e}ncio, Cha Zhang, and Michael Seltzer,
\newblock ``{CrowdMOS}: An approach for crowdsourcing mean opinion score
  studies,''
\newblock in {\em IEEE International Conference on Acoustics, Speech and Signal
  Processing (ICASSP)}. IEEE, 2011, pp. 2416--2419.

\end{thebibliography}

\end{document}